\documentclass[10pt]{article} 
\usepackage[preprint]{tmlr}
\usepackage{tikz}
\usetikzlibrary{positioning,shapes,arrows,calc,fit}

\usepackage{amsmath,amsfonts,bm}

\def\eqref#1{equation~\ref{#1}}

\def\1{\bm{1}}

\DeclareMathAlphabet{\mathsfit}{\encodingdefault}{\sfdefault}{m}{sl}
\SetMathAlphabet{\mathsfit}{bold}{\encodingdefault}{\sfdefault}{bx}{n}

\usepackage{hyperref}
\usepackage{url}
\usepackage{algorithm}
\usepackage{amsthm}
\usepackage{booktabs}
\usepackage{xcolor}
\usepackage{amssymb}
\usepackage{threeparttable}
\usepackage{microtype}
\usepackage{graphicx}
\usepackage{amsmath}
\usepackage{subfigure}
\usepackage[shortlabels]{enumitem}
\usepackage{booktabs} 
\usepackage{multirow} 
\usepackage{placeins}
\usepackage{amsthm}
\usepackage{natbib}
\newtheorem{example}{Example}
\usetikzlibrary{positioning, arrows.meta, shapes.geometric, shapes.misc, fit, calc, backgrounds}

\usepackage{hyperref}
\usepackage{float}
\usepackage{tikz}
\usetikzlibrary{arrows.meta,positioning,fit,backgrounds,calc}
\usepackage{amsmath}
\usepackage{amssymb}
\usepackage{mathtools}
\usepackage{amsthm}

\usepackage[capitalize,noabbrev]{cleveref}

\theoremstyle{plain}
\newtheorem{theorem}{Theorem}[section]
\newtheorem{proposition}[theorem]{Proposition}
\newtheorem{lemma}[theorem]{Lemma}
\newtheorem{corollary}[theorem]{Corollary}
\theoremstyle{definition}
\newtheorem{definition}[theorem]{Definition}

\theoremstyle{remark}
\newtheorem{remark}[theorem]{Remark}

\title{Real vs. Complex Spectral Bases for Neural Operators: The Role of Green's Function Alignment\thanks{Code available at \url{https://github.com/jaysulk/hartley-neural-operator}}}

\author{\name Jason Sulskis \email jason.sulskis@gtri.gatech.edu \\
\addr Department of Computer Science\\
      University of Illinois at Chicago\\[0.5em]
      \addr Applied Systems Laboratory, Systems Integration Division\\
      Georgia Tech Research Institute
      \AND
      \name Sathya Ravi \email sathya@uic.edu  \\
      \addr Department of Computer Science\\
      University of Illinois at Chicago}

\def\month{MM}  
\def\year{YYYY} 
\def\openreview{\url{https://openreview.net/forum?id=XXXX}} 

\begin{document}

\maketitle

\begin{abstract}
We introduce the Hartley Neural Operator (HNO), a real-valued counterpart of
the Fourier Neural Operator (FNO), which learns PDE solution operators using
global spectral convolutions. For real-valued input fields, Fourier coefficients
are conjugate symmetric, so standard FNO implementations retain a
half-spectrum with complex-valued multipliers. HNO instead uses the real
Discrete Hartley Transform and one real multiplier per retained mode.
Because Hartley coefficients do not have the same conjugate-symmetry
redundancy, HNO retains twice as many spectral corners; the two operators are
therefore isoparametric at equal width and differ primarily in spectral basis.
The central thesis is that the appropriate basis is a property of the
operator itself. For real, translation-invariant self-adjoint elliptic inverses, the
Green's kernel is even and its Fourier symbol is real and even; a
Hartley-diagonal layer can represent this multiplier exactly. This
alignment favors HNO on the periodic elliptic benchmarks. Transport introduces nonzero imaginary Fourier symbols, while coupled
wave evolution requires more than an uncoupled scalar diagonal multiplier;
FNO is therefore favored, with an advantage that increases with phase
content. The phaseless heat equation is the borderline case. Training both
operators with identical settings across seven PDEs, three initial-condition
families, and periodic and Dirichlet boundary conditions, we observe this
ordering, consistent with the Green's-function analysis. The result is a
predictive rule rather than a universal ranking: match the spectral basis to
the symmetry and phase structure of the solution operator.
\end{abstract}

\section{Introduction}
\label{introduction}

\paragraph{Why PDE solvers remain challenging.}
Partial differential equations (PDEs) govern fundamental physical processes, including seismic-wave propagation, heat diffusion, and quantum evolution. Classical numerical methods---including finite differences \citep{leveque2007}, finite elements \citep{johnson2009}, and spectral methods \citep{canuto2006}---have enabled major advances, but can face trade-offs in accuracy, stability, and computational cost, particularly in complex, high-dimensional, or real-time settings. Parabolic equations such as the heat equation smooth discontinuities and attenuate high frequencies, whereas hyperbolic equations such as the wave equation preserve information but introduce dispersive errors in numerical approximations, with consequences for seismology, acoustics, and electromagnetics. Nonlinear equations such as Burgers' equation can develop steep gradients and shocks that challenge fixed-basis representations, while the Navier--Stokes equations involve vorticity transport and turbulent cascades across scales. Elliptic equations, including Poisson and biharmonic problems, impose global coupling because their solutions depend on the entire domain.

Moreover, traditional solvers generally require a new computation for every initial condition or parameter choice, making repeated solves expensive in uncertainty quantification, inverse problems, and real-time forecasting \citep{pathak2022}. Neural operators \citep{li2020b,lu2021,kovachki2023} address this bottleneck by learning solution operators, amortizing computational cost across families of PDE problems and enabling rapid inference after training.

Neural operators learn mappings directly between function spaces. Among them, the Fourier Neural Operator (FNO) \citep{li2020b} has been particularly successful, parameterizing convolutional kernels in the frequency domain using the Fast Fourier Transform. Its architecture lifts the input to a higher-dimensional representation, applies truncated Fourier convolutions that retain selected frequency modes, and projects the result to the output space. Because these convolutions are implemented spectrally, FNOs can generalize across discretizations and resolutions. This approach is grounded in operator universal approximation theory \citep{chen1995}. Related architectures include the branch--trunk design of DeepONet \citep{lu2021} and wavelet-based operators \citep{tripura2022,gupta2021}, which provide improved time--frequency localization for multiscale phenomena.

Despite the prevalence of real-valued PDE solutions, the suitability of complex-valued Fourier representations as spectral bases has received comparatively little attention. For real fields, Fourier coefficients exhibit conjugate symmetry, and standard FNO implementations typically retain only one half of the spectrum while using complex-valued multipliers. We introduce the Hartley Neural Operator (HNO), which replaces the FFT with the real discrete Hartley transform \citep{hartley1942,bracewell1983}. Hartley coefficients are real and avoid the conjugate-symmetry redundancy of Fourier coefficients, allowing HNOs to retain all frequency quadrants without complex arithmetic. For symmetric kernels, the Hartley convolution theorem reduces convolution to an elementwise product. Consequently, an even real-valued Hartley multiplier can represent the corresponding operator exactly; more general phase-carrying operators require the omitted coupling between the $k$ and $-k$ modes.

A systematic comparison across PDE classes, three initial-condition families, and both periodic and Dirichlet boundary conditions, using identical training protocols for both operators, yields four main findings.
\begin{enumerate}
    \item \textbf{A real-valued counterpart to the FNO.}
    HNO uses one real multiplier per retained mode and requires no even--odd splitting. Because it retains twice as many spectral corners as the FNO, the two models have the same number of real spectral parameters at equal width. This shared real-weight parameterization follows \citet{wong2023hartleymha,wong2025hnoseg} and isolates the spectral basis as the primary architectural difference.

    \item \textbf{A basis--operator alignment principle.}
    For real, translation-invariant, self-adjoint elliptic inverses, the Green's kernel is even and the Fourier symbol is real and even, placing the inverse in the subspace of even real Hartley multipliers. They can therefore be represented exactly using one even real multiplier per mode. Conversely, this subspace excludes operators with phase-carrying symbols (Appendix~\ref{app:theory}).

    \item \textbf{An operator-phase ordering that predicts the empirical split.}
    The benchmark results exhibit a clear division: HNO performs best on elliptic operators, whereas FNO performs best on time-dependent operators. This division is consistent with differences in spectral phase and channel coupling. The biharmonic benchmark shows the largest HNO advantage; advection and Burgers' equation, which are the most transport-dominated cases, favor FNO most strongly; and the phaseless heat equation lies near the boundary. The initial-condition family changes the magnitude of these effects but not their direction.

    \item \textbf{Benchmarks and cost analysis.}
    We provide Burgers' and two-dimensional Navier--Stokes vorticity benchmarks for Hartley-based operators, together with an arithmetic-cost analysis and radix-4 FHT measurements. The results show that the additional wall-clock time arises from the emulated-DHT backend rather than from an inherent cost of the Hartley transform (Appendix~\ref{app:cost}).
\end{enumerate}

The takeaway is a guideline rather than a universal ranking: use the real Hartley basis when the solution operator is (approximately) self-adjoint and phase-free, as in many elliptic and diffusive problems; use the complex Fourier basis when the operator involves substantial oscillation or transport.

\section{Background}

\paragraph{The Discrete Hartley Transform.}
The discrete Hartley transform (DHT) \citep{bracewell1983} of a length-$N$ sequence $f[n]$ is
\begin{equation}
    H_k
    =
    \sum_{n=0}^{N-1}
    f[n]\,
    \operatorname{cas}\!\left(\frac{2\pi kn}{N}\right),
    \qquad
    \operatorname{cas}(\theta)=\cos(\theta)+\sin(\theta).
    \label{eq:dht}
\end{equation}
The DHT is self-inverse: applying it twice recovers the original sequence, up to the normalization factor $1/N$. For real-valued inputs, the DHT is related to the discrete Fourier transform (DFT) by
\begin{equation}
    H\{f\}(k)
    =
    \operatorname{Re}\{F\{f\}(k)\}
    -
    \operatorname{Im}\{F\{f\}(k)\}.
    \label{eq:hartley_fourier_relation}
\end{equation}
This identity underlies our implementation: we compute the DHT using \texttt{torch.fft} and extract the real part minus the imaginary part, thereby obtaining GPU-accelerated Hartley transforms without a custom kernel.

The Hartley convolution theorem \citep{bracewell1984} states that, for real signals $x$ and $y$ with Hartley spectra $X$ and $Y$, the Hartley spectrum of their circular convolution $x\circledast y$ is
\begin{equation}
    Z_k
    =
    \frac{1}{2}
    \left[
        X_k(Y_k+Y_{-k})
        +
        X_{-k}(Y_k-Y_{-k})
    \right],
    \label{eq:hartley_conv}
\end{equation}
where $Y_{-k}=Y_{N-k\bmod N}$. Unlike Fourier convolution, which is represented by elementwise complex multiplication, Hartley convolution couples the $k$ and $-k$ modes. For symmetric filters satisfying $Y_k=Y_{-k}$, however, this expression reduces to
\begin{equation}
    Z_k=X_kY_k.
\end{equation}
This reduction is particularly relevant to elliptic PDEs, whose Green's functions are real and symmetric under the conditions stated in Theorem~\ref{thm:elliptic_greens} (Appendix~\ref{app:theory}).

\paragraph{Comparison with other transforms.}
The Hartley transform occupies a conceptual position between the FFT and localized transforms such as wavelets \citep{daubechies1988} and the discrete cosine transform (DCT) \citep{ahmed1974}. Like the FFT, it provides a global frequency decomposition with $\mathcal{O}(N\log N)$ complexity, but it uses only real-valued coefficients. Wavelets provide time--frequency localization that is useful for transient and multiscale signals, but require a choice of mother wavelet and can be shift-variant. The DCT provides strong energy compaction for smooth data, although discontinuities can produce blocking artifacts. For smooth, periodic, real-valued PDE solutions, the Hartley transform therefore offers a natural representation without complex arithmetic.

Across all six PDEs in our preliminary experiments, the Multiwavelet Transform operator (MWT) \citep{tripura2022,gupta2021} consistently underperformed both FNO and HNO, with errors larger by factors of $2$--$4$. This result is consistent with the broader observation that smooth, periodic solutions with substantial global frequency content are often better represented by global spectral bases than by compactly supported wavelets. Appendix~\ref{app:transforms} provides a detailed comparison of the Hartley transform with wavelet and cosine alternatives.

Other recent studies have explored alternative spectral choices for neural operators. The Walsh--Hadamard Neural Operator \citep{cavallazzi2025whno} uses rectangular wave functions suited to discontinuous coefficients, whereas Convolutional Neural Operators \citep{raonic2023cno} avoid explicit spectral representations in favor of bandlimited convolutions. Together, these approaches reinforce the view that spectral-basis selection is an active architectural choice with meaningful performance implications.

In parallel, \citet{wong2023hartleymha,wong2025hnoseg} adapted Hartley-based neural operators to three-dimensional medical-image segmentation. Their results show that replacing the FFT with the DHT enables real-valued frequency-domain processing and achieves robust performance across resolutions with fewer than $35$k parameters. Their model uses a single shared real-valued weight per frequency mode. Our HNO adopts the same parameterization: one real multiplier per retained mode rather than separate even--odd weight pairs for each quadrant. Our theoretical analysis, based on Green's-function symmetry and operator phase content, explains why Hartley-based operators perform well for some operators but not for others. Accordingly, our comparison of FNO and HNO focuses on the central distinction between real and complex spectral representations.

\paragraph{Data-driven neural-operator training.}
We train HNO using supervised input--output pairs generated by numerical solvers. Because this approach learns the solution operator directly from examples, it avoids differentiating through the PDE and the optimization difficulties associated with balancing multiple loss terms \citep{wang2020}. For nonlinear time-dependent PDEs, we use finite-difference updates; for heat and wave we use analytical modal solutions. The numerical reference method is reported for each PDE below, rather than assumed to be bias-free. For periodic Poisson and biharmonic equations, we use spectral inversion on nonzero modes; this is exact up to discretization and numerical precision under the periodic assumptions. Dirichlet problems use the corresponding boundary-adapted formulation.

This setup isolates the effect of the spectral basis: performance differences between HNO and FNO primarily reflect how well each transform matches the PDE structure, rather than artifacts of the data-generation procedure.

\paragraph{Matched training for a basis-only comparison.}
Because HNO and FNO are isoparametric at equal width, we train both using the same optimizer, learning rate, schedule, weight decay, gradient clipping, and width. This protocol ensures that differences in accuracy are attributable to the spectral basis rather than to hyperparameter tuning. An earlier overparameterized Hartley variant used separate even--odd weight pairs for every frequency quadrant and appeared to require a separate training recipe. We attribute this behavior to redundant parameterization. In contrast, the real-diagonal HNO trains stably under the FNO protocol, making matched hyperparameters both fair and sufficient.

\section{Experiments}

We compare the Hartley Neural Operator (HNO) with the Fourier Neural Operator (FNO) \citep{li2020b,li2020a} on canonical PDEs from five classes: parabolic (heat), hyperbolic (wave), advective (advection--diffusion), nonlinear (Burgers' and two-dimensional Navier--Stokes in vorticity form), and elliptic (Poisson and biharmonic). To separate basis alignment from data-distribution effects, we consider three initial-condition families. We also evaluate both periodic and homogeneous Dirichlet boundary conditions. At equal width, the two operators are isoparametric and are trained with the same protocol (Section~\ref{sec:hpo}) using PyTorch \citep{paszke2019}.

\subsection{Initial-Condition Families}

We evaluate three initial-condition families to distinguish basis alignment from data-distribution effects, following the comparison protocol of \citet{lu2022}.

\paragraph{Gaussian random fields (GRFs).}
We sample stochastic initial conditions spectrally from a Matérn covariance \citep{matern1986,gneiting2010} with $\nu=2.5$, $\ell=0.15$, and $\sigma=1.0$. These fields have broadband spectral content with power-law decay
\[
    S(k)\propto
    \left(1+\ell^2|k|^2\right)^{-(\nu+d/2)},
\]
and model settings such as turbulence and stochastic forcing. Their broadband content provides a stringent test of both spectral representations across the retained mode range.

\paragraph{Eigenfunction expansions.}
We generate superpositions of Laplacian eigenfunctions with Sobolev weights \citep{canuto2006}:
\[
    u_0(x,y)
    =
    \sum_{k,\ell}
    \hat{u}_{k\ell}
    (1+k^2+\ell^2)^{-s}
    \sin(\pi kx)\sin(\pi \ell y).
\]
Because these initial conditions are composed of Fourier modes, they provide a setting in which the FNO basis is structurally aligned with the data.

\paragraph{Gaussian-bump superpositions.}
We construct localized initial conditions from 2--5 Gaussian bumps with centers sampled uniformly from $[0.2,0.8]^2$, widths $\sigma\in[0.06,0.15]$, and amplitudes $a\in[-1,1]$. Their spectra decay rapidly, concentrating most of their energy in a small number of low-frequency modes.

\subsection{Ground-Truth Generation}

We evaluate time-dependent PDEs at 51 time points over $t\in[0,0.5]$ on a $64\times64$ grid, following \citet{lu2022}. For heat and for the wave displacement with zero initial velocity, the exact Fourier-space solutions are
\[
    \widehat{u}(k,t)
    =
    \widehat{u}_0(k)e^{-\nu|k|^2t},
    \qquad
    \widehat{u}(k,t)
    =
    \widehat{u}_0(k)\cos(c|k|t),
\]
respectively. The wave cosine expression assumes zero initial velocity;
for nonzero initial velocity a $\sin(c|k|t)/(c|k|)$ contribution is required.
These formulas avoid time-discretization error when their stated initial data apply. Because Burgers' equation and Navier--Stokes lack closed-form solutions, we generate their ground truth with finite-difference methods, avoiding an implicit bias toward Fourier-based representations. Burgers' equation uses first-order upwind advection, central diffusion, and adaptive CFL subcycling. Navier--Stokes uses upwind vorticity transport, a spectral streamfunction solve, and adaptive subcycling.

We solve the advection--diffusion equation with an upwind/spectral transport scheme using advection speed $c_{\mathrm{adv}}=1.5$ and diffusivity $\nu_{\mathrm{adv}}=0.01$, under both boundary conditions. Elliptic problems are solved spectrally, with
\[
    \widehat{u}_{k,\ell}
    =
    \frac{\widehat{f}_{k,\ell}}
    {4\pi^2(k^2+\ell^2)}
\]
for Poisson's equation and
\[
    \widehat{u}_{k,\ell}
    =
    \frac{\widehat{f}_{k,\ell}}
    {16\pi^4(k^2+\ell^2)^2}
\]
for the biharmonic equation.

Each time-dependent PDE is evaluated at three parameter settings:
\[
\begin{aligned}
    \nu_{\mathrm{heat}} &\in \{0.005,0.01,0.05\},&
    c_{\mathrm{wave}} &\in \{0.5,1.0,2.0\},\\
    \nu_{\mathrm{Burgers}} &\in \{0.01,0.02,0.05\},&
    \nu_{\mathrm{NS}} &\in \{0.001,0.005,0.01\}.
\end{aligned}
\]
After normalizing each solution to unit maximum amplitude, we generate 200 samples per configuration, using 160 for training and 40 for testing.

\subsection{Network Architectures}

\paragraph{FNO and HNO.}
For a controlled comparison, FNO and HNO use the same architectural skeleton based on \citet{li2020b}. For time-dependent PDEs, each network consists of an input projection, three spectral-convolution blocks with residual $1\times1$ convolutions and GELU activations \citep{hendrycks2023}, and an output projection. Elliptic problems use four spectral-convolution blocks.

For time-dependent problems, the input has four channels,
\[
    [u_0,x,y,t],
\]
where $u_0$ is broadcast over the spatial grid and $t$ is the target time broadcast over the grid. The network predicts $u(\cdot,t)$ for each of the $N_t=51$ target times independently, treating every $(u_0,t)$ pair as a separate training example. For elliptic problems, the input has three channels, $[f,x,y]$.

\begin{figure}[t]
\centering
\resizebox{\linewidth}{!}{%
\begin{tikzpicture}[
  font=\sffamily\scriptsize,
  >={Stealth[length=2mm,width=1.4mm]},
  node distance=7mm,
  io/.style={draw=black!60,fill=black!5,rounded corners=2pt,
    minimum height=8mm,align=center,inner xsep=4pt},
  shared/.style={draw=blue!65!black,fill=blue!7,rounded corners=2pt,
    minimum height=9mm,align=center,inner xsep=5pt},
  spec/.style={draw=teal!75!black,fill=teal!9,rounded corners=2pt,
    minimum height=8mm,align=center,inner xsep=5pt},
  weight/.style={draw=orange!85!black,fill=orange!12,rounded corners=2pt,
    minimum height=8mm,align=center,inner xsep=5pt},
  arrow/.style={-Stealth,thick,draw=black!75},
  panel/.style={draw=black!35,rounded corners=3pt,fill=black!1,inner sep=6pt},
  label/.style={font=\sffamily\scriptsize\bfseries,align=left}]
\node[io] (in) at (0,0) {Input\\$[u_0,x,y,t]$ / $[f,x,y]$};
\node[shared, right=9mm of in] (lift) {Lift\\$P$};
\node[shared, right=9mm of lift,minimum width=35mm] (block) {Residual spectral block\\$\sigma\!\left(\mathcal{K}_{\mathcal B}(v)+W_{1\times1}v\right)$};
\node[shared, right=9mm of block] (proj) {Projection\\$Q$};
\node[io, right=9mm of proj] (out) {Prediction\\$u(\cdot,t)$ / $u(\cdot)$};
\draw[arrow] (in)--(lift);\draw[arrow](lift)--(block);
\draw[arrow](block)--node[above,font=\tiny]{repeat $L$}(proj);
\draw[arrow](proj)--(out);
\node[font=\tiny,below=2mm of block] {$L=3$ (evolution), $L=4$ (elliptic); $\sigma=\mathrm{GELU}$};
\node[label,above=4mm of block] {Shared network (same width, blocks, and training)};
\node[label,anchor=west] at (-.2,-20mm) {Only the spectral branch $\mathcal K_{\mathcal B}$ changes};
\node[spec] (fwdF) at (1.0,-33mm) {\textbf{FNO}\\$\mathcal F$ (\texttt{rFFTn})};
\node[weight, right=9mm of fwdF] (wF) {Complex $R(k)$\\retained half-spectrum};
\node[spec, right=9mm of wF] (invF) {$\mathcal F^{-1}$\\real field};
\node[io, right=11mm of invF] (resF) {Spectral output\\$\mathcal K_{\mathcal F}(v)$};
\draw[arrow](fwdF)--(wF);\draw[arrow](wF)--(invF);\draw[arrow](invF)--(resF);
\node[spec] (fwdH) at (1.0,-50mm) {\textbf{HNO}\\$\mathcal H$ (DHT)};
\node[weight, right=9mm of fwdH] (wH) {Real $W(k)$\\all retained corners};
\node[spec, right=9mm of wH] (invH) {$\mathcal H^{-1}$\\real field};
\node[io, right=11mm of invH] (resH) {Spectral output\\$\mathcal K_{\mathcal H}(v)$};
\draw[arrow](fwdH)--(wH);\draw[arrow](wH)--(invH);\draw[arrow](invH)--(resH);
\begin{scope}[on background layer]
\node[panel,fit=(fwdF)(wF)(invF)(resF)]{};
\node[panel,fit=(fwdH)(wH)(invH)(resH)]{};
\end{scope}
\node[font=\tiny,align=center] at (6.3,-63mm)
 {Matched real parameter count: 2 complex corners $\leftrightarrow$ 4 real corners (2D).\\
  Hartley diagonal multipliers are not general Hartley convolutions: phase requires $k/-k$ coupling.};
\end{tikzpicture}%
}
\caption{Matched FNO and HNO architectures. Both use the same input lift,
$L$ residual spectral blocks with pointwise $1\times1$ branches and GELU,
and output projection. Each spectral branch applies a forward transform,
learned modewise channel mixing, and an inverse transform. FNO stores complex
weights on the retained real-FFT half-spectrum; HNO stores real weights at
all retained Hartley corners. Their real spectral parameter counts match
at fixed channel width, but a diagonal Hartley multiplier cannot encode a
general phase-carrying convolution without coupling opposite modes.}
\label{fig:hno_arch}
\end{figure}

\paragraph{FNO spectral convolution.}
FNO implements spectral convolution with the real FFT, \texttt{rFFTn}, which exploits conjugate symmetry in the transform of a real-valued input by retaining only one half of the spectrum \citep{oppenheim1999}:
\begin{equation}
    \mathrm{FNO}: \qquad
    y
    =
    \mathcal{F}^{-1}
    \left[
        R\cdot\mathcal{F}[x]
    \right].
    \label{eq:fno_conv}
\end{equation}
Here $R$ denotes the complex channel-mixing weights on the retained modes (including the two stored spectral corners in 2D); the displayed one-dimensional mode index is schematic. The inverse transform reconstructs the full spectrum using conjugate symmetry.

\paragraph{HNO spectral convolution.}
Although the Hartley transform of a real field is real, its coefficients generally do not satisfy Fourier conjugate symmetry; $H_k$ and $H_{-k}$ are independent. HNO therefore retains all low-frequency corners: four in two dimensions and eight in three:
\begin{equation}
    \mathrm{HNO}: \qquad
    y
    =
    \mathcal{H}^{-1}
    \left[
        W\cdot\mathcal{H}[x]
    \right],
    \qquad
    W\in\mathbb{R}^{d_v\times d_v\times k_{\max}}\quad\text{(schematic mode index)}.
    \label{eq:hno_conv}
\end{equation}
For symmetric kernels, Hartley convolution reduces to an elementwise product. An even real multiplier can therefore represent the corresponding operator exactly, and the unconstrained HNO contains this representation as a special case. By omitting the $k/-k$ coupling, this parameterization cannot represent phase-carrying operators within the even-multiplier subspace (Appendix~\ref{app:phase}).

A complex weight contains two real parameters, whereas a real weight contains one. HNO compensates for this difference by retaining twice as many spectral corners as FNO: it uses the full spectrum rather than the half-spectrum returned by \texttt{rFFT}. Thus, at equal width, the two models have the same number of real spectral parameters. In two dimensions, FNO's two complex corners and HNO's four real corners each contribute
\[
    4d_v^2m^2
\]
real weights. In three dimensions, the corresponding count is
\[
    8d_v^2m^3
\]
for both models. We therefore set
\[
    d_v^{\mathrm{FNO}}=d_v^{\mathrm{HNO}},
\]
making the comparison isoparametric and restricting the architectural difference to the spectral basis.

\subsection{Training Protocol}
\label{sec:hpo}

To isolate the effect of the spectral basis, we train both operators with identical configurations. We use Adam with learning rate $10^{-3}$, weight decay $10^{-4}$, gradient clipping at norm $1.0$, and a step schedule that halves the learning rate after each quarter of training. The models also use the same channel width and number of retained modes per axis. Because HNO and FNO are isoparametric at equal width, this choice matches both their widths and their numbers of real trainable parameters.

We use the standard per-sample relative-$L^2$ loss, which avoids the constant-output collapse that unnormalized MSE can induce for small-amplitude, zero-mean targets. We perform no method-specific hyperparameter search. The real-diagonal HNO trains stably with the FNO settings, making matched hyperparameters both fair and sufficient.

\subsection{Evaluation Metrics}

Following standard neural-operator evaluation practice \citep{li2020b,lu2022}, we report the relative $L^2$ error averaged over test samples:
\begin{equation}
    \mathrm{Rel}\text{-}L^2
    =
    \frac{1}{N_{\mathrm{test}}}
    \sum_{i=1}^{N_{\mathrm{test}}}
    \frac{
        \left\|u_\theta^{(i)}-u_{\mathrm{true}}^{(i)}\right\|_2
    }{
        \left\|u_{\mathrm{true}}^{(i)}\right\|_2
    }.
    \label{eq:rel_l2}
\end{equation}

To assess spatial derivative accuracy, which is important for quantities such as fluxes, stresses, and forces \citep{kovachki2023}, we also report the gradient error:
\begin{equation}
    \mathrm{GradError}
    =
    \frac{
        \left\|\nabla u_\theta-\nabla u_{\mathrm{true}}\right\|_2
    }{
        \left\|\nabla u_{\mathrm{true}}\right\|_2
    }.
    \label{eq:grad_error}
\end{equation}
Gradients are computed using central finite differences on the output grid. We additionally report per-sample error distributions to assess consistency beyond mean performance.

\subsection{Computational Setup}

All experiments were conducted on a single NVIDIA A100 GPU with 40~GB of memory through Google Colab Pro. Training required 20--40 minutes per method for time-dependent PDEs and 3--5 minutes per method for elliptic PDEs.

\section{Results}

We compare HNO and FNO across the PDE suite, three initial-condition families, and both boundary-condition settings. The operators are isoparametric and trained with identical protocols. We report relative $L^2$ errors and find the predicted separation between elliptic and time-dependent problems, consistent with the operator-symmetry theory in Appendix~\ref{app:theory}.

\subsection{FNO versus HNO: Overall Performance}

\begin{figure}[H]
    \centering
    \includegraphics[width=0.8\columnwidth]{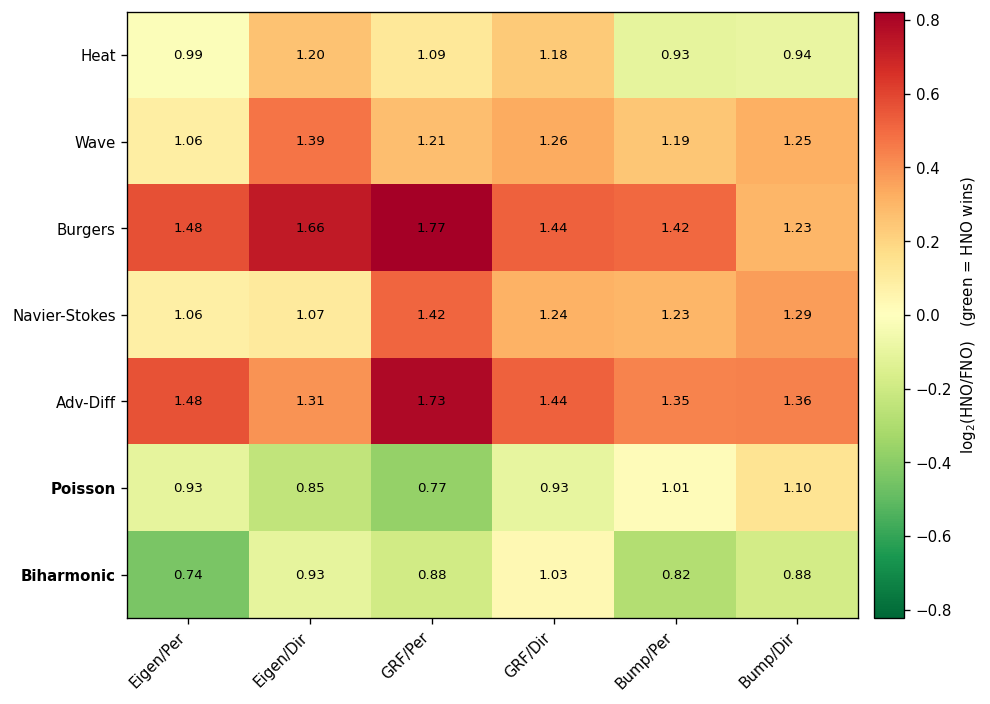}
    \caption{
        Ratio of HNO to FNO relative-$L^2$ error across PDEs (rows) and
        initial-condition/boundary-condition combinations (columns).
        Green indicates lower HNO error, red lower FNO error, and bold rows
        denote elliptic operators. The elliptic rows are predominantly green,
        whereas the time-dependent rows are predominantly red. The FNO
        advantage generally increases with phase content, from the phaseless
        heat equation to transport-dominated operators.
    }
    \label{fig:heatmap}
\end{figure}

The results separate primarily according to operator symmetry and vary
monotonically with phase content rather than according to a simple
time-dependent versus elliptic dichotomy (Figure~\ref{fig:heatmap}).

\paragraph{Elliptic PDEs favor HNO.}
For Poisson and biharmonic equations, whose operators are self-adjoint and
possess real, symmetric Green's functions, HNO achieves lower relative
$L^2$ error. Its largest advantage occurs for the biharmonic equation. This
result has a direct theoretical basis: Corollary~\ref{cor:symmetric_conv}
shows that a single real multiplier per mode can represent a symmetric
kernel exactly. The HNO hypothesis class therefore contains the target
operator, while FNO can represent the same target using real-valued complex weights;
the extra imaginary parameters are unnecessary for that target.
The advantage is stronger under periodic boundary conditions, for which the
operator is exactly diagonal in the Hartley basis, than under homogeneous
Dirichlet conditions.

\begin{figure}[t]
    \centering
    \includegraphics[width=0.75\linewidth]{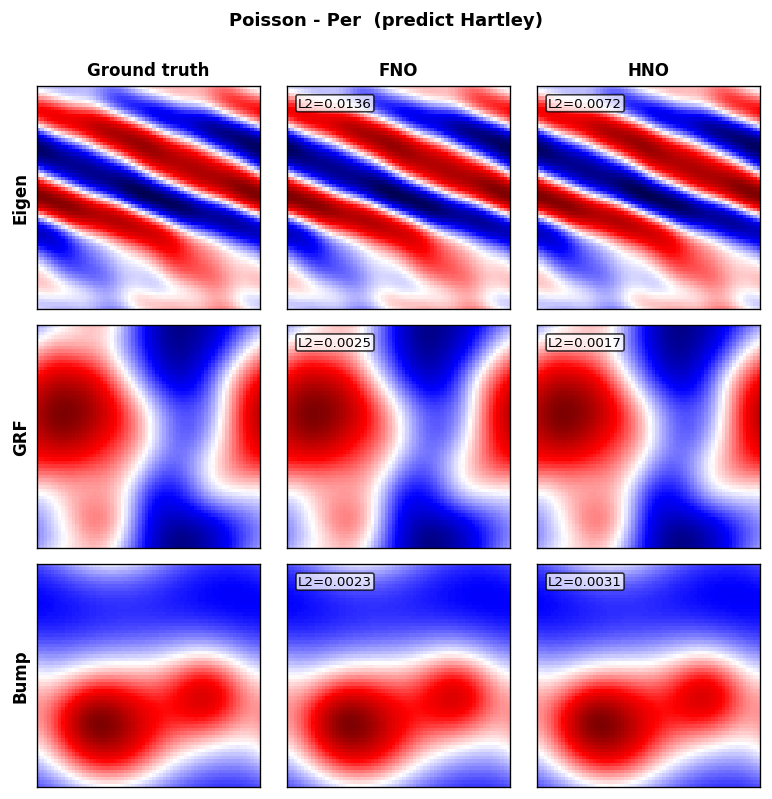}
    \caption{
        Qualitative comparison for the periodic Poisson problem. Rows show
        predictions for the eigenfunction, Gaussian random field (GRF), and
        Gaussian-bump initial-condition families, respectively. Columns show
        the ground truth, FNO prediction, and HNO prediction. The relative
        $L^2$ errors are reported in the upper-left corner of each prediction
        panel. HNO provides the more accurate prediction for the eigenfunction
        and GRF inputs, while both operators perform comparably for the
        Gaussian-bump input.
    }
    \label{fig:poisson_comparison}
\end{figure}

Figure~\ref{fig:poisson_comparison} provides a qualitative comparison for the
periodic Poisson problem. The predictions from both operators reproduce the
large-scale spatial structure of the ground truth, but HNO more closely
matches the reference for the eigenfunction and GRF inputs, consistent with
its lower relative-$L^2$ errors. For the smoother Gaussian-bump input, the
two predictions are visually similar and their errors are comparable. This
example illustrates that HNO's advantage on elliptic problems is not limited
to a particular initial-condition family, although its magnitude depends on
the spectral content of the input.

\paragraph{Time-dependent PDEs favor FNO, with an ordering by phase content.}
For the heat, wave, advection--diffusion, Burgers, and Navier--Stokes
equations, FNO generally achieves lower error, and its advantage increases
with the phase content of the solution operator
(Corollary~\ref{cor:phase_ordering}). The largest gains occur for
transport-dominated problems, including advection--diffusion and Burgers'
equation. Although the wave equation is oscillatory, it contains less
transport-induced phase structure and occupies an intermediate position.
The heat propagator,
\[
    e^{-\nu |k|^2t},
\]
is real and phase-free, making the heat equation the closest comparison
between the two operators; HNO occasionally performs better for especially
smooth inputs. This ordering supports the proposed mechanism: the relative
strength of FNO reflects the phase carried by the propagator rather than
the conventional classification of the PDE.

\paragraph{Initial conditions modulate the magnitude, not the direction, of the effect.}
Across the eigenfunction, GRF, and Gaussian-bump families, the preferred
operator remains largely unchanged, while the initial-condition family
modulates the size of the performance gap. The smooth Gaussian-bump family
produces localized, low-frequency solutions that both operators represent
effectively. Its ratios are therefore closest to one and provide the least
diagnostic separation between the bases.

The aggregate ratios in Figure~\ref{fig:heatmap} quantify the magnitude of
each operator's advantage but do not indicate where that advantage occurs.
By resolving the error spatially, Figure~\ref{fig:spatial_4panel} shows that
the two regimes differ not only in magnitude but also in spatial structure.
For elliptic and phaseless problems, HNO's advantage is diffuse and
typically low-amplitude: it is uniformly, although modestly, closer to the
reference solution across most of the domain. For example, in the
biharmonic case, HNO is preferred at approximately $72\%$ of pixels with a
representative ratio of $0.82$, whereas in the heat case it is preferred
across the full domain with a ratio of $0.93$. Thus, a ratio close to one in
these settings need not indicate a pointwise tie; it can reflect a modest
advantage distributed across the domain.

Transport-dominated problems show the opposite behavior. FNO's advantage
is concentrated around moving structures, including the steepened fronts in
Burgers' equation and the advective streaks in the advection--diffusion
equation, while the surrounding smooth regions remain close to parity.
This spatial pattern is consistent with the phase-content hypothesis: the
real Hartley basis naturally represents the diffuse, phase-free structure
of self-adjoint operators, but is less effective for localized,
phase-carrying features that complex Fourier multipliers can encode.

\begin{figure}[t]
    \centering
    \includegraphics[width=\linewidth]{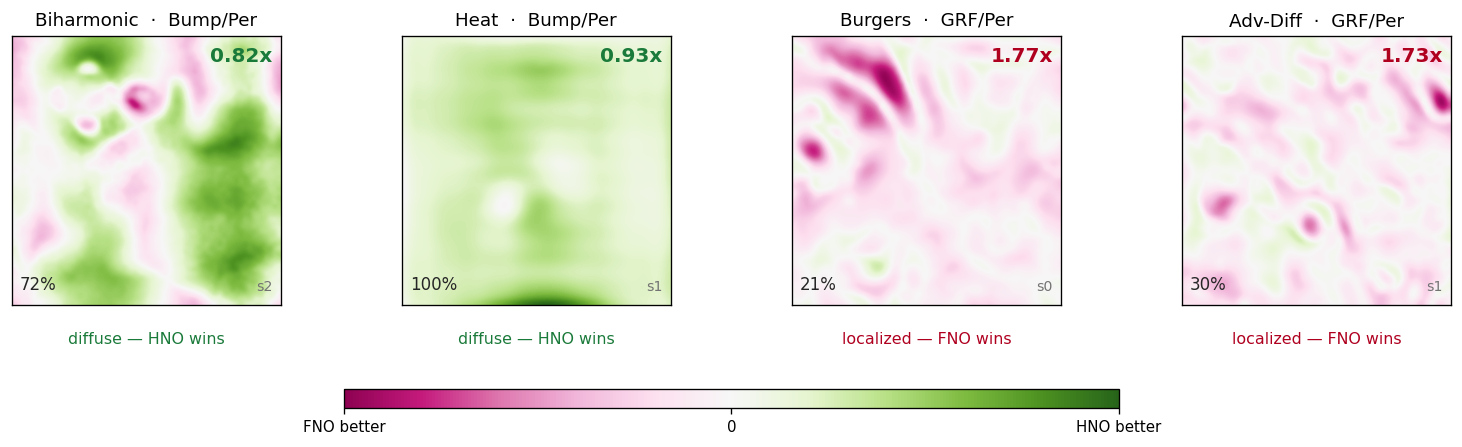}
    \caption{
        \textbf{The spatial location of each operator's advantage is
        structured.}
        Per-pixel error advantage
        $|u_{\mathrm{FNO}}-u|-|u_{\mathrm{HNO}}-u|$ for four representative
        cases. Green indicates that HNO is closer to the reference solution,
        whereas pink indicates that FNO is closer. The top-right panel shows
        the relative-$L^2$ ratio (HNO/FNO), so values below one favor HNO.
        The bottom-left panel reports the fraction of pixels at which HNO is
        closer. For elliptic and phaseless cases (left), HNO's advantage is
        diffuse and spread across the domain at low amplitude. For
        transport-dominated cases (right), FNO's advantage is localized on
        advective structures, including Burgers' steepened fronts and
        advection--diffusion streaks. These structures carry the phase that a
        real-diagonal operator cannot represent. Panels show the best-of-seed run per cell.
    }
    \label{fig:spatial_4panel}
\end{figure}

\section{Discussion}

Our experiments suggest a single organizing principle: the symmetry and phase
structure of the solution operator should guide the choice of spectral basis.
Real, self-adjoint, phase-free operators are naturally matched to the real
Hartley basis, whereas operators with substantial phase content are better
matched to the complex Fourier basis. The following subsections discuss the
mechanisms underlying this division, its computational implications, and its
relationship to other real-valued operator representations.

\subsection{Green's-Function Alignment Explains Elliptic Dominance}

HNO's advantage on elliptic PDEs follows from the symmetry of their Green's
functions. Under the assumptions of Theorem~\ref{thm:elliptic_greens}, a
self-adjoint elliptic operator has a real, symmetric Green's function. In
translation-invariant settings, this corresponds to a real and even spectral
symbol; the Fourier and Hartley representations then encode the same
operator. For symmetric kernels, Hartley convolution reduces to
elementwise multiplication (Corollary~\ref{cor:symmetric_conv}), placing the
target in HNO's even real-multiplier subspace. An unconstrained FNO instead
uses complex-valued multipliers and therefore introduces additional imaginary
degrees of freedom (Proposition~\ref{prop:complexity}) that are unnecessary
for this operator class.

For the periodic, translation-invariant elliptic inverses, this diagonalization
is exact. Homogeneous Dirichlet boundary conditions break translation
invariance: the corresponding Laplacian inverse is diagonal in a sine basis,
not generally in the Hartley basis. HNO's Dirichlet gains are empirical,
not a consequence of exact Hartley diagonalization.

\subsection{Phase Content Explains the Time-Dependent Ordering}

The graded FNO advantage on time-dependent PDEs is consistent with the phase
content of their solution operators. An even real Hartley multiplier can
represent only operators with real, even symbols
(Proposition~\ref{prop:phaseless}). The full first-order wave propagator couples displacement and velocity
through oscillatory sine and cosine blocks, although the displacement-only
map with zero initial velocity has the real, even multiplier $\cos(c|k|t)$.
Advection instead contributes the complex phase factor
\[
    e^{-i c\cdot k\,t}.
\]
These contributions lie outside the even real-multiplier subspace and favor
the complex Fourier representation. The advantage therefore increases with
the operator's phase content (Corollary~\ref{cor:phase_ordering}). The heat
equation is a borderline case because its propagator,
\[
    e^{-\nu |k|^2t},
\]
is real and even. Initial-condition families affect the magnitude of the
performance gap but generally not its direction, suggesting that the
operator symbol, rather than the input spectrum alone, is the primary
determinant of the preferred basis.

\subsection{Computational Cost}

The real-diagonal HNO layer uses one real matrix--vector multiplication per
retained spectral mode and does not require complex arithmetic. At the
spectral-layer level, this gives a simpler parameterization and lower
per-mode arithmetic than a complex FNO multiplier. In the present
implementation, however, the DHT is evaluated as
\[
    \operatorname{Re}\{\mathcal{F}[x]\}
    -
    \operatorname{Im}\{\mathcal{F}[x]\}
\]
using a full complex \texttt{fftn}. This implementation does not exploit the
real-arithmetic structure of a native fast Hartley transform and therefore
incurs additional wall-clock cost relative to the \texttt{rfft} backend used
by FNO.

This overhead is a property of the emulated-DHT backend rather than an
intrinsic limitation of the Hartley transform. A native fast Hartley
transform based on real arithmetic could substantially reduce this cost.
Indeed, the radix-4 FHT measurements in Appendix~\ref{app:cost},
Table~\ref{tab:fht_benchmark}, indicate considerable optimization headroom,
with measured speedups of $40$--$131\times$ over the particular radix-2
reference implementation; these are not speedups over an optimized FFT. Relative to the accuracy differences induced by
basis selection, the observed overhead remains modest.

\subsection{Related Real-Basis Operators}

Recent work on Walsh--Hadamard operators for discontinuous coefficients
\citep{cavallazzi2025whno} and convolutional neural operators
\citep{raonic2023cno} further supports the view that basis selection is an
active architectural choice. The Walsh--Hadamard basis is suited to
piecewise or discontinuous structure, whereas convolutional neural operators
avoid an explicit spectral basis altogether.

Independent evidence comes from the HNOSeg line of work
\citep{wong2023hartleymha,wong2025hnoseg}. In three-dimensional medical-image
segmentation, Hartley-based operators achieve strong accuracy and
resolution robustness with fewer than $35$k parameters, substantially fewer
than the parameter counts of conventional CNN- and transformer-based
alternatives. Their observation that real-valued frequency representations
permit stacked nonlinear operations complements our results. In our setting,
the real-valued spectral diagonal exactly matches real, symmetric solution
operators, while its limitations become apparent for operators whose symbols
carry nonzero phase.

\subsection{Limitations}

Our experiments use square domains at a resolution of $64\times64$ and
consider periodic and homogeneous Dirichlet boundary conditions. These
settings may not capture all effects that arise at higher resolutions, on
irregular domains, or under more general boundary conditions. The
viscosities used for Burgers' and Navier--Stokes equations also limit the
formation of fully developed shocks and turbulence. Finally, our HNO
implementation computes the DHT through \texttt{torch.fft} rather than a
native FHT kernel, so the reported runtime does not represent the full
efficiency potential of a real-valued Hartley implementation.

\section{Conclusion}

We introduced the Hartley Neural Operator (HNO), a real-valued counterpart
to the Fourier Neural Operator (FNO) that uses one real multiplier per
retained mode and an isoparametric parameterization at equal width. Our
results support a principle for basis selection rather than a universal
ranking of the two operators.

Real, translation-invariant self-adjoint elliptic inverses have even
Green's kernels and real, even symbols, which an even Hartley multiplier
represents exactly. This aligns with HNO's largest measured advantage on
the biharmonic benchmark. By contrast, this
subspace cannot represent nonzero phase without additional $k/-k$ coupling,
so FNO performs better on oscillatory and transport-dominated operators.
Across our benchmarks, the FNO advantage generally increases from the
phase-free heat equation through wave and transport-dominated problems.
Initial-condition families and boundary conditions primarily affect the
magnitude of these differences rather than their direction.

The practical guideline is therefore to match the spectral basis to the
symmetry and phase structure of the solution operator. Real, even, phase-free
operators are natural candidates for a Hartley representation, whereas
oscillatory and advective operators are better suited to a complex Fourier
representation. Future work should investigate native differentiable FHT
kernels, data-driven per-mode basis selection, and comparisons with
nonspectral architectures such as DeepONet \citep{lu2021} and convolutional
neural operators \citep{raonic2023cno}.

\bibliography{tmlr}

@article{li2020b,
  title={Fourier Neural Operator for Parametric Partial Differential Equations},
  author={Li, Zongyi and Kovachki, Nikola and Azizzadenesheli, Kamyar and Liu, Burigede and Bhattacharya, Kaushik and Stuart, Andrew and Anandkumar, Anima},
  journal={arXiv preprint arXiv:2010.08895},
  year={2020}
}

@inproceedings{li2020a,
  title={Fourier Neural Operator for Parametric Partial Differential Equations},
  author={Li, Zongyi and Kovachki, Nikola and Azizzadenesheli, Kamyar and Liu, Burig and Bhattacharya, Kirthevasan and Stuart, Andrew and Anandkumar, Anima},
  booktitle={Proceedings of the 37th International Conference on Machine Learning (ICML)},
  year={2020}
}

@article{lu2021,
  title={Learning nonlinear operators via {DeepONet} based on the universal approximation theorem of operators},
  author={Lu, Lu and Jin, Pengzhan and Pang, Guofei and Zhang, Zhongqiang and Karniadakis, George Em},
  journal={Nature Machine Intelligence},
  volume={3},
  number={3},
  pages={218--229},
  year={2021},
  publisher={Nature Publishing Group}
}

@article{kovachki2023,
  title={Neural operator: Learning maps between function spaces with applications to {PDEs}},
  author={Kovachki, Nikola and Li, Zongyi and Liu, Burigede and Azizzadenesheli, Kamyar and Bhattacharya, Kaushik and Stuart, Andrew and Anandkumar, Anima},
  journal={Journal of Machine Learning Research},
  volume={24},
  number={89},
  pages={1--97},
  year={2023}
}

@article{lu2022,
  title={A comprehensive and fair comparison of two neural operators (with practical extensions) based on {FAIR} data},
  author={Lu, Lu and Meng, Xuhui and Cai, Shengze and Mao, Zhiping and Goswami, Somdatta and Zhang, Zhongqiang and Karniadakis, George Em},
  journal={Computer Methods in Applied Mechanics and Engineering},
  volume={393},
  pages={114778},
  year={2022},
  publisher={Elsevier}
}

@article{chen1995,
  title={Universal approximation to nonlinear operators by neural networks with arbitrary activation functions and its application to dynamical systems},
  author={Chen, Tianping and Chen, Hong},
  journal={IEEE Transactions on Neural Networks},
  volume={6},
  number={4},
  pages={911--917},
  year={1995},
  publisher={IEEE}
}

@article{tripura2022,
  title={Wavelet neural operator for solving parametric partial differential equations in computational mechanics problems},
  author={Tripura, Tapas and Chakraborty, Souvik},
  journal={Computer Methods in Applied Mechanics and Engineering},
  volume={404},
  pages={115783},
  year={2022},
  publisher={Elsevier}
}

@inproceedings{gupta2021,
  title={Multiwavelet-based Operator Learning for Differential Equations},
  author={Gupta, Gaurav and Xiao, Xiongye and Bogdan, Paul},
  booktitle={Advances in Neural Information Processing Systems},
  editor={A. Beygelzimer and Y. Dauphin and P. Liang and J. Wortman Vaughan},
  year={2021},
  url={https://openreview.net/forum?id=LZDiWaC9CGL}
}

@article{hartley1942,
  title={A More Symmetrical {Fourier} Analysis Applied to Transmission Problems},
  author={Hartley, R.V.L.},
  journal={Proceedings of the IRE},
  volume={30},
  number={3},
  pages={144--150},
  year={1942},
  doi={10.1109/JRPROC.1942.234333}
}

@article{bracewell1983,
  title={Discrete {Hartley} Transform},
  author={Bracewell, R. N.},
  journal={Journal of the Optical Society of America},
  volume={73},
  number={12},
  pages={1832--1835},
  year={1983},
  doi={10.1364/JOSA.73.001832}
}

@book{bracewell1984,
  title={The {Hartley} Transform},
  author={Bracewell, R. N.},
  publisher={Oxford University Press},
  year={1984}
}

@article{ahmed1974,
  title={Discrete Cosine Transform},
  author={Ahmed, N. and Natarajan, T. and Rao, K. R.},
  journal={IEEE Transactions on Computers},
  volume={100},
  number={1},
  pages={90--93},
  year={1974}
}

@article{daubechies1988,
  title={Orthonormal bases of compactly supported wavelets},
  author={Daubechies, Ingrid},
  journal={Communications on Pure and Applied Mathematics},
  volume={41},
  number={7},
  pages={909--996},
  year={1988}
}

@book{akansu1992,
  title={Multiresolution Signal Decomposition: Transforms, Subbands, and Wavelets},
  author={Akansu, Ali N and Haddad, Richard A},
  year={1992},
  publisher={Academic Press}
}

@book{kaiser1994,
  title={A Friendly Guide to Wavelets},
  author={Kaiser, Gerald},
  year={1994},
  publisher={Birkh\"{a}user}
}

@article{wallace1991,
  title={The {JPEG} still picture compression standard},
  author={Wallace, Gregory K},
  journal={Communications of the ACM},
  volume={34},
  number={4},
  pages={30--44},
  year={1991}
}

@book{oppenheim1999,
  title={Discrete-Time Signal Processing},
  author={Oppenheim, Alan V. and Schafer, Ronald W.},
  edition={2nd},
  publisher={Prentice Hall},
  year={1999}
}

@book{leveque2007,
  title={Finite Difference Methods for Ordinary and Partial Differential Equations},
  author={LeVeque, Randall J.},
  year={2007},
  publisher={SIAM}
}

@book{johnson2009,
  title={Numerical Solution of Partial Differential Equations by the Finite Element Method},
  author={Johnson, Claes},
  year={2009},
  publisher={Dover Publications}
}

@book{canuto2006,
  title={Spectral Methods: Fundamentals in Single Domains},
  author={Canuto, Claudio and Hussaini, M. Yousuff and Quarteroni, Alfio and Zang, Thomas A.},
  year={2006},
  publisher={Springer}
}

@book{matern1986,
  title={Spatial Variation},
  author={Mat\'{e}rn, Bertil},
  year={1986},
  publisher={Springer New York},
  doi={10.1007/978-1-4615-7892-5}
}

@article{gneiting2010,
  title={Mat\'{e}rn Cross-Covariance Functions for Multivariate Random Fields},
  author={Gneiting, Tilmann and Kleiber, William and Schlather, Martin},
  journal={Journal of the American Statistical Association},
  volume={105},
  number={491},
  pages={1167--1177},
  year={2010},
  doi={10.1198/jasa.2010.tm09420}
}

@article{paszke2019,
  title={{PyTorch}: An Imperative Style, High-Performance Deep Learning Library},
  author={Paszke, Adam and Gross, Sam and Massa, Francisco and Lerer, Adam and Bradbury, James and Chanan, Gregory and Killeen, Trevor and Lin, Zeming and Gimelshein, Natalia and Antiga, Luca and others},
  journal={Advances in Neural Information Processing Systems},
  volume={32},
  year={2019}
}

@misc{hendrycks2023,
  title={Gaussian Error Linear Units ({GELUs})},
  author={Hendrycks, Dan and Gimpel, Kevin},
  year={2023},
  eprint={1606.08415},
  archivePrefix={arXiv},
  primaryClass={cs.LG}
}

@misc{pathak2022,
  title={{FourCastNet}: A Global Data-driven High-resolution Weather Model using Adaptive {Fourier} Neural Operators},
  author={Pathak, Jaideep and Subramanian, Shashank and Harrington, Peter and Raja, Sanjeev and Chattopadhyay, Ashesh and Mardani, Morteza and Kurth, Thorsten and Hall, David and Li, Zongyi and Azizzadenesheli, Kamyar and Hassanzadeh, Pedram and Kashinath, Karthik and Anandkumar, Animashree},
  year={2022},
  eprint={2202.11214},
  archivePrefix={arXiv},
  primaryClass={physics.ao-ph}
}

@article{wang2020,
  title={Understanding and Mitigating Gradient Flow Pathologies in Physics-informed Neural Networks},
  author={Wang, Sifan and Teng, Yibo and Perdikaris, Paris},
  journal={SIAM Journal on Scientific Computing},
  volume={42},
  number={5},
  pages={A2927--A2950},
  year={2020},
  publisher={SIAM}
}

@article{cavallazzi2025whno,
  title={Walsh-Hadamard Neural Operators for Solving {PDEs} with Discontinuous Coefficients},
  author={Cavallazzi, Giorgio M. and P\'{e}rez Cuadrado, Miguel and Pinelli, Alfredo},
  journal={arXiv preprint arXiv:2511.07347},
  year={2025}
}

@article{raonic2023cno,
  author = {B.~Raoni\'c},
  title = {Convolutional neural operators for robust and accurate learning of PDEs},
  journal = {NeurIPS},
  year = {2023},
}

@article{wong2025hnoseg,
  title={{HNOSeg-XS}: Extremely Small {Hartley} Neural Operator for Efficient and Resolution-Robust {3D} Image Segmentation},
  author={Wong, Ken C. L. and Wang, Hongzhi and Syeda-Mahmood, Tanveer},
  journal={IEEE Transactions on Medical Imaging},
  year={2025},
  doi={10.1109/TMI.2025.3588458}
}

@inproceedings{wong2023hartleymha,
  title={{HartleyMHA}: Self-Attention in Frequency Domain for Resolution-Robust and Parameter-Efficient {3D} Image Segmentation},
  author={Wong, Ken C. L. and Wang, Hongzhi and Syeda-Mahmood, Tanveer},
  booktitle={International Conference on Medical Image Computing and Computer-Assisted Intervention (MICCAI)},
  year={2023}
}
\bibliographystyle{tmlr}

\appendix
\section{Notation and Definitions}
\label{app:notation}

We establish notation for the spectral transforms used throughout this work. All definitions below apply to real-valued functions and signals.

\begin{definition}[Fourier Transform]
\label{def:fourier}
For $f \in L^1(\mathbb{R}^n) \cap L^2(\mathbb{R}^n)$, the continuous Fourier transform is
\begin{equation}
    \mathcal{F}\{f\}(k) = \int_{\mathbb{R}^n} f(x)\, e^{-ik \cdot x}\, dx.
    \label{eq:ft}
\end{equation}
Given a sequence $f[n]$ of length $N$, the Discrete Fourier Transform (DFT) is
\begin{equation}
    X_k = \sum_{n=0}^{N-1} f[n]\, e^{-i\, 2\pi kn/N} = \sum_{n=0}^{N-1} f[n] \left[\cos\!\left(\frac{2\pi kn}{N}\right) - i\sin\!\left(\frac{2\pi kn}{N}\right)\right].
    \label{eq:dft}
\end{equation}
\end{definition}

\begin{definition}[Hartley Transform]
\label{def:hartley}
Define the cas (cosine-and-sine) kernel as
\begin{equation}
    \mathrm{cas}(\theta) = \cos(\theta) + \sin(\theta) = \sqrt{2}\,\sin\!\left(\theta + \frac{\pi}{4}\right).
    \label{eq:cas}
\end{equation}
The continuous Hartley transform of $f$ is
\begin{equation}
    \mathcal{H}\{f\}(k) = \int_{\mathbb{R}^n} f(x)\, \mathrm{cas}(k \cdot x)\, dx.
    \label{eq:cht}
\end{equation}
Up to transform normalization, the Hartley transform is involutory;
using the standard continuous Fourier convention, the inverse is
\begin{equation}
    f(x) = \frac{1}{(2\pi)^n}\int_{\mathbb{R}^n} \mathcal{H}\{f\}(k)\, \mathrm{cas}(k \cdot x)\, dk.
    \label{eq:cht_inv}
\end{equation}
Given a sequence $f[n]$ of length $N$, the Discrete Hartley Transform (DHT) \citep{bracewell1983} is
\begin{equation}
    H_k = \sum_{n=0}^{N-1} f[n] \cdot \mathrm{cas}\!\left(\frac{2\pi kn}{N}\right) = \sum_{n=0}^{N-1} f[n] \left[\cos\!\left(\frac{2\pi kn}{N}\right) + \sin\!\left(\frac{2\pi kn}{N}\right)\right].
    \label{eq:dht_app}
\end{equation}
The DHT is a linear map $\mathbb{R}^N \to \mathbb{R}^N$. Its inverse is obtained by applying the DHT again and normalizing by $1/N$, so the DHT is self-inverse up to a scale factor.
\end{definition}

\begin{definition}[Transform Relationship]
\label{def:transform_rel}
For real-valued $f$, the two transforms are related by
\begin{equation}
    \mathcal{H}\{f\}(k) = \mathrm{Re}\{\mathcal{F}\{f\}(k)\} - \mathrm{Im}\{\mathcal{F}\{f\}(k)\}.
    \label{eq:rel}
\end{equation}
This identity is the basis for our GPU-accelerated implementation: we compute the DHT via \texttt{torch.fft.fftn} and extract $\mathrm{Re} - \mathrm{Im}$.
\end{definition}

\paragraph{Hartley convolution.}
For the cyclic convolution of vectors $x$ and $y$ producing $z$ (all length $N$), let $X$, $Y$, and $Z$ denote their respective DHTs. Then
\begin{align}
    Z_k &= \frac{X_k(Y_k + Y_{N-k}) + X_{N-k}(Y_k - Y_{N-k})}{2}, \label{eq:hconv1} \\
    Z_{N-k} &= \frac{X_{N-k}(Y_k + Y_{N-k}) - X_k(Y_k - Y_{N-k})}{2}. \label{eq:hconv2}
\end{align}

\section{Comparison with Wavelet and Cosine Transforms}
\label{app:transforms}

Integral transforms, including the wavelet transform \citep{daubechies1988}, the Hartley transform, and the discrete cosine transform (DCT) \citep{ahmed1974}, enable data analysis and representation in transformed domains. In the wavelet case, basis functions are generated by dilating and translating a “mother wavelet” \citep{akansu1992}, whereas the Hartley transform relies on cas functions (Definition~\ref{def:hartley}). The DCT, built on cosine functions, excels at concentrating signal energy into a small number of coefficients—a trait exploited in JPEG compression \citep{wallace1991}—yet it suffers from blocking artifacts at discontinuities and, unlike the Hartley transform, is not self-inverse.

Wavelet and Hartley transforms both involve integrating the original signal against a kernel: the wavelet function \citep{kaiser1994} for wavelets and the cas function for Hartley. Wavelets are particularly adept at capturing transient events and multiscale phenomena, but they demand careful selection of a mother wavelet, exhibit shift variance, and cannot be readily operated upon in the frequency domain. In contrast, the Hartley transform’s real-valued formulation, which couples sinusoidal and cosinusoidal components, yields representations well suited to smooth, periodic signals whose spectral content is globally distributed.

Empirically, the Multiwavelet Transform operator (MWT) \citep{tripura2022,gupta2021} consistently underperformed both FNO and HNO by factors of 2–4$\times$ across all six PDEs tested. This outcome supports the view that smooth, periodic solutions with global frequency structure are better captured by spectral bases than by compactly supported wavelets.

\section{The Fast Hartley Transform}
\label{app:fht}

The Fast Hartley Transform (FHT) computes the DHT in $\mathcal{O}(N \log N)$ operations using divide-and-conquer, analogous to the Cooley-Tukey FFT. The radix-2 algorithm proceeds as follows:

\begin{enumerate}
    \item \textbf{Decompose:} Split the input into even- and odd-indexed elements:
    \begin{equation}
        x_\mathrm{even}[n] = x[2n], \quad x_\mathrm{odd}[n] = x[2n+1].
        \label{eq:fht_split}
    \end{equation}
    \item \textbf{Recurse:} Apply the DHT to these smaller sequences:
    \begin{equation}
        \mathcal{H}(x_\mathrm{even})[k], \quad \mathcal{H}(x_\mathrm{odd})[k], \quad k = 0, \ldots, \tfrac{N}{2}-1.
        \label{eq:fht_recurse}
    \end{equation}
    \item \textbf{Combine:} A Hartley butterfly couples opposite frequencies.
    With $E_k$ and $O_k$ the length-$N/2$ Hartley transforms of the even
    and odd subsequences, $\theta_k=2\pi k/N$, and indices interpreted
    modulo $N/2$, the radix-2 combination is
    \begin{align}
        \mathcal{H}(x)[k] &= E_k+\cos(\theta_k)O_k+\sin(\theta_k)O_{-k},\label{eq:fht_combine1}\\
        \mathcal{H}(x)[k+N/2] &= E_k-\cos(\theta_k)O_k-\sin(\theta_k)O_{-k}.\label{eq:fht_combine2}
    \end{align}
    for $k=0,\ldots,N/2-1$. The $O_{-k}$ term cannot in general be
    replaced by a scalar $\operatorname{cas}(\theta_k)O_k$ factor.
\end{enumerate}

Higher-radix variants (radix-4, split-radix) reduce the multiplicative constant by processing more elements per recursion level. The FHT retains all DHT properties---linearity, symmetry, Parseval's theorem---with reduced computational complexity suitable for real-time applications.

\paragraph{Implementation note.}
We do not use a standalone FHT kernel in our experiments. Instead, we compute
the DHT from a full complex FFT using
\[
\mathcal{H}\{f\}(k)
=
\operatorname{Re}\{\mathcal{F}\{f\}(k)\}
-
\operatorname{Im}\{\mathcal{F}\{f\}(k)\},
\]
with the precise scaling determined by the transform convention. This
implementation uses \texttt{torch.fft.fftn} and therefore benefits from
cuFFT's optimized GPU kernels, but it does not exploit the potential
real-arithmetic savings of a native FHT. The reported numerical results agree
with those of the corresponding DHT up to normalization; the additional
wall-clock cost is attributable to the emulated-DHT backend rather than to
the Hartley representation itself.

\section{Arithmetic Cost Comparison}
\label{app:cost}

We compare the arithmetic cost of na\"{i}ve (non-FFT) forward and inverse transforms.

\paragraph{Na\"{i}ve DFT (forward + inverse):}
Each $(k,n)$ pair requires 2 real multiplications and 2 real additions (one for the real part, one for imaginary), giving 4 flops per pair. Total: $C_\mathrm{DFT} = 2 \times 4N^2 = 8N^2$.

\paragraph{Na\"{i}ve DHT (forward + inverse):}
Each $(k,n)$ pair requires 1 real multiplication and 1 real addition, giving 2 flops per pair. Total: $C_\mathrm{Hartley} = 2 \times 2N^2 = 4N^2$.

\paragraph{Comparison:}
\begin{equation}
    \frac{C_\mathrm{Hartley}}{C_\mathrm{DFT}} = \frac{4N^2}{8N^2} = \frac{1}{2}.
    \label{eq:cost_ratio}
\end{equation}

For the fast algorithms ($\mathcal{O}(N \log N)$), the ratio depends on implementation details but remains favorable for the FHT due to purely real arithmetic. In practice, when computing the DHT via \texttt{torch.fft} as in our implementation, the overhead relative to a direct FFT call is minimal---a single elementwise $\mathrm{Re} - \mathrm{Im}$ operation on the FFT output.

\subsection{Radix-4 FHT SpectralConv3d Performance}

We benchmark a radix-4 Fast Hartley Transform implementation against a standard radix-2 FHT within the SpectralConv3d layer used in our neural operator architecture. The radix-4 variant processes four elements per recursion level rather than two, reducing recursion depth from $\log_2(N)$ to $\log_4(N) = \log_2(N)/2$.

\begin{table}[h]
\centering
\caption{Performance comparison: SpectralConv3d using radix-4 FHT vs.\ radix-2 FHT. The radix-4 implementation achieves 40--131$\times$ speedup with favorable scaling properties. FHT $\phi$ and IFHT $\phi$ denote the forward and inverse transform phase timings within the radix-4 implementation.}
\label{tab:fht_benchmark}
\begin{tabular}{lccccc}
\toprule
Input Size & Radix-4 $t$ (s) & Radix-2 $t$ (s) & Speedup & FHT $\phi$ (s) & IFHT $\phi$ (s) \\
\midrule
$16^3$ & $0.0079 \pm 0.0001$ & $0.3168 \pm 0.0078$ & 40.09$\times$ & 0.0037 & 0.0037 \\
$32^3$ & $0.0340 \pm 0.0020$ & $3.4990 \pm 0.5385$ & 102.93$\times$ & 0.0164 & 0.0165 \\
$64^3$ & $0.2516 \pm 0.0106$ & $32.9607 \pm 0.7072$ & 131.03$\times$ & 0.1219 & 0.1221 \\
\bottomrule
\end{tabular}
\end{table}

The speedup factor increases from approximately 40$\times$ to over 131$\times$ as input size grows, demonstrating favorable scaling:
\begin{equation}
    \frac{d(\mathrm{Speedup})}{dN} > 0 \quad \text{for } N \in [16^3, 64^3].
\end{equation}
The timing breakdown reveals excellent load balancing between forward and inverse transform phases ($t_\mathrm{FHT} \approx t_\mathrm{IFHT}$), indicating symmetric implementation of both transform directions.

The radix-4 advantage stems from three factors: (i) halved recursion depth reduces function call overhead and improves instruction-level parallelism; (ii) each butterfly combines four elements, increasing arithmetic intensity per memory access; and (iii) the larger processing block improves spatial and temporal cache locality. Both algorithms maintain $\mathcal{O}(N \log N)$ asymptotic complexity, but the radix-4 constant factors are substantially smaller:
\begin{align}
    T_\mathrm{radix\text{-}4}(N) &= C_1 \cdot N \log_4(N) + \mathcal{O}(N), \\
    T_\mathrm{radix\text{-}2}(N) &= C_2 \cdot N \log_2(N) + \mathcal{O}(N),
\end{align}
where $C_1 \ll C_2$ empirically.

\paragraph{Note on our experimental implementation.}
The radix-4 FHT benchmarks above characterize standalone transform performance. In the experiments reported in this paper, we compute the DHT via \texttt{torch.fft} (Section~\ref{app:fht}) rather than a native FHT kernel, as this leverages cuFFT's highly optimized GPU implementation with minimal additional overhead. A native radix-4 FHT integrated into PyTorch's autograd system is a promising direction for further reducing HNO's computational cost, particularly for large-scale 3D problems where the 131$\times$ speedup over radix-2 would translate to significant wall-clock savings.

\section{Theoretical Analysis: Spectral Basis Selection}
\label{app:theory}

This appendix explains why HNO is favored for elliptic operators with real,
symmetric Green's functions, whereas FNO is favored for time-dependent
operators with phase-carrying propagators. We analyze Green's-function
symmetry, spectral parameterization, phase representation, and truncation.

\subsection{Preliminaries}

We use the Fourier and Hartley transforms defined in
Definitions~\ref{def:fourier}--\ref{def:hartley}, together with the
relationship in \eqref{eq:rel}. Let $\mathcal{L}$ be a linear differential
operator on $\Omega\subseteq\mathbb{R}^n$ with appropriate boundary
conditions. We consider
\[
\mathcal{L}u=f,
\qquad
f\in L^2(\Omega),
\]
with $u$ having the required Sobolev regularity.

\begin{definition}[Green's function]
\label{def:greens}
The Green's function
$G:\Omega\times\Omega\to\mathbb{R}$ satisfies
\[
\mathcal{L}_xG(x,y)=\delta(x-y)
\]
with the boundary conditions imposed on $u$. For translation-invariant
operators on $\mathbb{R}^n$ or periodic domains,
$G(x,y)=G(x-y)$.
\end{definition}

\begin{definition}[Self-adjoint operator]
An operator $\mathcal{L}$ is self-adjoint if
\[
\langle\mathcal{L}u,v\rangle_{L^2}
=
\langle u,\mathcal{L}v\rangle_{L^2}
\]
for all $u,v\in\operatorname{dom}(\mathcal{L})$.
\end{definition}

\begin{definition}[Spectral neural operator]
A spectral neural operator with basis
$\mathcal{B}\in\{\mathcal{F},\mathcal{H}\}$ has the form
\begin{equation}
(\mathcal{K}_{\theta}v)(x)
=
\mathcal{B}^{-1}
\left[
W_{\theta}(k)\cdot\mathcal{B}\{v\}(k)
\right](x),
\label{eq:spectral_no}
\end{equation}
where $W_{\theta}$ is a learned spectral weight tensor.
\end{definition}

\subsection{Symmetry of Elliptic Green's Functions}

\begin{lemma}[Symmetry of self-adjoint Green's functions]
\label{lem:symmetry}
If $\mathcal{L}$ is invertible and self-adjoint on a real Hilbert space,
and its inverse admits a Green's kernel $G$, then
\[
G(x,y)=G(y,x).
\]
For translation-invariant problems, this implies
\[
G(z)=G(-z).
\]
\end{lemma}

\begin{proof}
Let $u$ and $v$ solve
\[
\mathcal{L}u=\delta_y,
\qquad
\mathcal{L}v=\delta_x.
\]
Self-adjointness gives
\[
\langle\delta_y,v\rangle
=
\langle\mathcal{L}u,v\rangle
=
\langle u,\mathcal{L}v\rangle
=
\langle u,\delta_x\rangle.
\]
Thus $v(y)=u(x)$, so $G(y,x)=G(x,y)$. If
$G(x,y)=G(x-y)$, then
\[
G(x-y)=G(y-x),
\]
which gives $G(z)=G(-z)$.
\end{proof}

\begin{proposition}[Reality of elliptic Green's functions]
If $\mathcal{L}$ has real coefficients and real boundary conditions, then its
Green's function is real-valued.
\end{proposition}

\begin{proof}
Complex conjugation gives
\[
\mathcal{L}\overline{G}=\delta.
\]
Uniqueness therefore implies $\overline{G}=G$.
\end{proof}

\begin{theorem}[Spectral structure of elliptic Green's functions]
\label{thm:elliptic_greens}
Let $G$ be the real, even convolution kernel of an invertible
translation-invariant self-adjoint elliptic operator, or of its inverse
on the subspace where the inverse is defined. Then
\[
\widehat{G}(k)\in\mathbb{R},
\qquad
\widehat{G}(k)=\widehat{G}(-k),
\qquad
\mathcal{H}\{G\}(k)=\widehat{G}(k).
\]
\end{theorem}

\begin{proof}
Because $G$ is even,
\[
\widehat{G}(k)
=
\int G(x)\cos(k\cdot x)\,dx
-
i\int G(x)\sin(k\cdot x)\,dx
\]
has a vanishing sine term. Hence $\widehat{G}(k)$ is real. Since $G$ is
real,
\[
\widehat{G}(-k)=\overline{\widehat{G}(k)}=\widehat{G}(k),
\]
so the Fourier symbol is even. The sine term also vanishes in the Hartley
transform, yielding
\[
\mathcal{H}\{G\}(k)=\widehat{G}(k).
\]
\end{proof}

\subsection{Explicit Green's-Function Spectra}

\begin{example}[Poisson equation]
For $-\nabla^2u=f$ on $\mathbb{R}^n$,
\[
|k|^2\widehat{G}(k)=1,
\qquad
\widehat{G}(k)=|k|^{-2},
\]
away from the zero mode.
\end{example}

\begin{example}[Biharmonic equation]
For $\nabla^4u=f$,
\[
|k|^4\widehat{G}(k)=1,
\qquad
\widehat{G}(k)=|k|^{-4}.
\]
\end{example}

\begin{example}[Heat equation]
For
\[
\partial_tu=\nu\nabla^2u,
\qquad
u(x,0)=u_0(x),
\]
the Fourier-space solution is
\[
\widehat{u}(k,t)
=
\widehat{u}_0(k)e^{-\nu|k|^2t}.
\]
The propagator $e^{-\nu|k|^2t}$ is real and even, so heat is the borderline
time-dependent case discussed in \cref{app:phase}.
\end{example}

\subsection{Spectral Parameterization}

\begin{proposition}[Spectral parameter count]
\label{prop:complexity}
For $M$ retained modes per channel pair, an unconstrained FNO uses $2M$
real-valued spectral parameters, whereas the HNO uses $M$. Thus, for a real,
even target symbol, the HNO uses half as many spectral parameters as the
corresponding unconstrained FNO parameterization.
\end{proposition}

\begin{proof}
FNO multipliers lie in $\mathbb{C}$ and require two real parameters per mode.
HNO multipliers lie in $\mathbb{R}$ and require one. For a real, even target
symbol, the HNO parameterization matches the target structure directly.
\end{proof}

\subsection{Hartley Convolution}

\begin{theorem}[Discrete Hartley convolution]
\label{thm:hconv}
Let $f,g\in\mathbb{R}^N$, and let $H_f$ and $H_g$ denote their DHTs. For the
circular convolution $h=f\circledast g$,
\begin{equation}
\begin{aligned}
H_h(k)
=
\frac{1}{2}\Big[
&H_f(k)H_g(k)
+H_f(k)H_g(-k)\\
&+H_f(-k)H_g(k)
-H_f(-k)H_g(-k)
\Big],
\end{aligned}
\label{eq:hconv}
\end{equation}
where indices are interpreted modulo $N$.
\end{theorem}

\begin{corollary}[Symmetric kernels convolve elementwise]
\label{cor:symmetric_conv}
If $g(x)=g(-x)$, then $H_g(k)=H_g(-k)$ and
\[
H_{f\circledast g}(k)=H_f(k)H_g(k).
\]
\end{corollary}

\begin{proof}
Substituting $H_g(-k)=H_g(k)$ into \cref{thm:hconv}, the last two terms
cancel and the first two combine to give $H_f(k)H_g(k)$.
\end{proof}

Thus, a real, even elliptic Green's function is represented exactly by an
even real Hartley multiplier. If the HNO assigns independent weights at
$k$ and $-k$, this exact representation remains available within the
corresponding even-weight subspace.

\subsection{Optimization Landscape}

Consider learning the Poisson symbol
\[
\widehat{G}_j=|k_j|^{-2}
\]
at $M$ retained modes. The FNO learns
\[
R_j=R_j^{\mathrm{re}}+iR_j^{\mathrm{im}}\in\mathbb{C},
\]
whereas the HNO learns $W_j\in\mathbb{R}$. Their losses are
\begin{align}
\mathcal{L}_{\mathrm{FNO}}
&=
\sum_{j=1}^{M}
\left[
\left(R_j^{\mathrm{re}}-\widehat{G}_j\right)^2
+
\left(R_j^{\mathrm{im}}\right)^2
\right],
\label{eq:loss_fno}\\
\mathcal{L}_{\mathrm{HNO}}
&=
\sum_{j=1}^{M}
\left(W_j-\widehat{G}_j\right)^2.
\label{eq:loss_hno}
\end{align}

\begin{theorem}[Minimum characterization]
The global minimizers are
\[
R_j^{\mathrm{re}}=\widehat{G}_j,
\qquad
R_j^{\mathrm{im}}=0
\]
for FNO, and
\[
W_j=\widehat{G}_j
\]
for HNO. Thus, the FNO optimization problem is posed in
$\mathbb{R}^{2M}$, whereas the HNO problem is posed in $\mathbb{R}^{M}$.
Within this reduced $M$-mode comparison, the FNO optimum lies on the
subspace $R^{\mathrm{im}}=0$. This toy objective does not compare the full
isoparametric HNO and FNO networks, which retain different mode sets.
\end{theorem}

\begin{proposition}[Optimization geometry]
\label{prop:convergence}
The FNO minimizer lies in $\mathbb{R}^{2M}$ with
$R_j^{\mathrm{re}}=\widehat{G}_j$ and $R_j^{\mathrm{im}}=0$, whereas the HNO
minimizer lies in $\mathbb{R}^{M}$ with $W_j=\widehat{G}_j$. If initialization
is centered at the respective optimum with variance $\sigma^2$ per real
parameter, then
\[
\mathbb{E}\|\theta_0-\theta^*\|^2
=
2M\sigma^2
\quad\text{for FNO},
\qquad
\mathbb{E}\|\theta_0-\theta^*\|^2
=
M\sigma^2
\quad\text{for HNO}.
\]
Under gradient flow for these quadratic losses,
\[
\|\theta_t-\theta^*\|
=
\|\theta_0-\theta^*\|e^{-2t}.
\]
In this reduced parameter-space model the FNO has an initial
distance larger by $\sqrt{2}$, but both quadratic objectives have the
same gradient-flow contraction rate. This is not a convergence-speed
claim for the matched, isoparametric neural operators.
\end{proposition}

\subsection{Real Multipliers and Phaseless Operators}
\label{app:phase}

The complement of the elliptic argument explains FNO's advantage on
time-dependent PDEs. The relevant quantity is the phase of the solution
operator's symbol, rather than the spectral content of the initial condition.

\begin{definition}[Operator symbol and phase]
For a translation-invariant solution operator $\mathcal{K}$ satisfying
\[
\widehat{\mathcal{K}u}(k)
=
\lambda(k)\widehat{u}(k),
\]
$\lambda(k)\in\mathbb{C}$ is the symbol of $\mathcal{K}$. The operator is
\emph{phaseless} if $\lambda(k)$ is real and even, and
\emph{phase-carrying} when its imaginary part is nonzero at some mode.
A negative real symbol changes sign without introducing a general transport phase.
\end{definition}

\begin{proposition}[Representation of phaseless symbols]
\label{prop:phaseless}
The even real Hartley-diagonal operator
\[
\mathcal{K}_W u
=
\mathcal{H}^{-1}
\left[
W\cdot\mathcal{H}\{u\}
\right],
\qquad
W(k)=W(-k)\in\mathbb{R},
\]
represents exactly the translation-invariant operators with real, even
Fourier symbols. Operators with
$\operatorname{Im}\lambda(k)\neq0$ cannot be represented within this
even-multiplier subspace.
\end{proposition}

\begin{proof}
For even real $W$, the Hartley convolution identity reduces to elementwise
multiplication, and the corresponding Fourier symbol is
$\lambda(k)=W(k)$. Conversely, a nonzero imaginary component of
$\lambda(k)$ cannot be produced by an even real multiplier.
\end{proof}

\begin{corollary}[Phase ordering of the benchmark PDEs]
\label{cor:phase_ordering}
The following symbol calculations motivate the observed ordering; they do not establish a universal accuracy ranking:
\begin{itemize}
\item \textbf{Heat:}
$\lambda=e^{-\nu|k|^2t}$ is real and even; HNO is in-class, making heat the
borderline case.
\item \textbf{Wave:} the first-order displacement--velocity propagator
contains sine and cosine coupling between channels, while the
zero-initial-velocity displacement map has the real, even cosine symbol.
An FNO advantage is an empirical result, not forced by a nonzero imaginary
scalar symbol in that special case.
\item \textbf{Advection, Burgers, and Navier--Stokes:} transport terms
contribute factors such as $e^{-ic\cdot kt}$ and their linearizations,
making these operators strongly phase-carrying; FNO is increasingly favored
as transport dominates.
\end{itemize}
\end{corollary}

For nonlinear PDEs, this interpretation applies to the local linearization.
For
\[
\mathcal{N}(u)
=
-u\cdot\nabla u+\nu\nabla^2u,
\]
the Fr\'echet derivative at $\overline{u}$ is
\[
D\mathcal{N}|_{\overline{u}}[\delta u]
=
-\overline{u}\cdot\nabla\delta u
-\delta u\cdot\nabla\overline{u}
+\nu\nabla^2\delta u.
\]
The advective terms contribute imaginary factors such as
\[
-ik'\widehat{\overline{u}}(k-k'),
\]
whereas diffusion contributes a real term proportional to $\nu|k|^2$.
Increasing viscosity should therefore reduce, but need not reverse, FNO's
advantage.

\subsection{Truncation Is Basis-Independent}

\begin{definition}[Truncation operator]
For $\mathcal{B}\in\{\mathcal{F},\mathcal{H}\}$ and cutoff $M$, define
\[
\Pi_M^{\mathcal{B}}f
=
\mathcal{B}^{-1}
\left[
\mathbf{1}_{|k|\leq M}\mathcal{B}\{f\}
\right].
\]
\end{definition}

\begin{theorem}[Truncation equivalence for real functions]
\label{thm:truncation}
Let $f\in H^s(\mathbb{T}^d)$ be real-valued with $s>d/2$. If the retained
Fourier and Hartley mode sets are symmetric under $k\mapsto-k$, then
\[
\|f-\Pi_M^{\mathcal{F}}f\|_{L^2}
=
\|f-\Pi_M^{\mathcal{H}}f\|_{L^2}.
\]
Moreover,
\[
\|f-\Pi_M^{\mathcal{B}}f\|_{L^2}
\leq
C_s\|f\|_{H^s}M^{-s},
\]
and
\[
\|\nabla(f-\Pi_M^{\mathcal{B}}f)\|_{L^2}
\leq
C_s\|f\|_{H^s}M^{-(s-1)}.
\]
\end{theorem}

\begin{proof}
For real $f$,
\[
\widehat{f}(-k)=\overline{\widehat{f}(k)}.
\]
Using the Fourier--Hartley relation, the cross terms cancel when the
contributions from $k$ and $-k$ are combined. Thus, the Fourier and Hartley
tails have equal energy. The stated rates follow from the standard Sobolev
tail estimates; differentiation reduces the rate by one power of $M$.
\end{proof}

\begin{remark}
\cref{thm:truncation} shows that the two bases have equal truncation
accuracy under symmetric mode retention. Their performance differences
therefore arise from operator realizability and optimization rather than
from the approximation power of the truncated transforms.
\end{remark}

\subsection{Learned Operator Error}

\begin{theorem}[Error decomposition]
\label{thm:error_decomposition}
Under standard approximation and generalization assumptions, a spectral
operator with basis $\mathcal{B}$, $M$ retained modes, learned weights
$W_{\theta}$, and optimal weights $W^*$ satisfies
\begin{equation}
\begin{aligned}
\frac{\|u_{\theta}-u_{\mathrm{true}}\|_{L^2}}
{\|u_{\mathrm{true}}\|_{L^2}}
&\leq
\underbrace{C_sM^{-s}}_{\text{truncation}}
+
\underbrace{\|W_{\theta}-W^*\|_{\mathrm{op}}}
_{\text{optimization and realizability}}\\
&\quad+
\underbrace{\mathcal{O}
\!\left(N_{\mathrm{train}}^{-1/2}\right)}
_{\text{generalization}}.
\end{aligned}
\label{eq:error_decomposition}
\end{equation}
\end{theorem}

\begin{remark}
The truncation term is basis-independent by
\cref{thm:truncation}. Hence, the primary FNO--HNO difference lies in
optimization and realizability. Phaseless operators lie in the even real
Hartley-multiplier subspace, whereas phase-carrying operators do not.
\end{remark}

\subsection{Extension to Nonlinear PDEs}

The preceding analysis applies directly to linear operators. For nonlinear
PDEs, we use local linearization.

\begin{theorem}[Local linearization for nonlinear PDEs]
\label{thm:local_linear}
Consider
\[
\partial_tu=\mathcal{N}(u)
\]
and let $D\mathcal{N}|_{\overline{u}}$ denote the Fr\'echet derivative at a
reference solution $\overline{u}$. Perturbations satisfy
\[
\partial_t\delta\widehat{u}(k)
=
\sum_{k'}
\widehat{D\mathcal{N}}(k,k')\delta\widehat{u}(k').
\]
The diffusive component is real and phaseless, whereas the advective
component is imaginary and phase-carrying. The expected basis preference is
therefore determined by the relative strength of advection and diffusion.
\end{theorem}

\begin{proof}
For Burgers' equation,
\[
\mathcal{N}(u)
=
-u\cdot\nabla u+\nu\nabla^2u,
\]
the linearization is
\[
D\mathcal{N}|_{\overline{u}}[\delta u]
=
-\overline{u}\cdot\nabla\delta u
-\delta u\cdot\nabla\overline{u}
+\nu\nabla^2\delta u.
\]
In Fourier space, the advective terms contain factors of the form
\[
-ik'\widehat{\overline{u}}(k-k'),
\]
whereas diffusion contributes a real diagonal term proportional to
$\nu|k|^2$. Thus, the imaginary component increases with advection strength,
while the real diffusive component increases with viscosity.
\end{proof}

\begin{corollary}[Predicted nonlinear-PDE behavior]
\label{cor:nonlinear_prediction}
For Burgers and Navier--Stokes at the viscosities studied here, the
phase-carrying advective component is expected to dominate, favoring FNO.
Increasing viscosity should narrow the performance gap but need not reverse
the ordering, since these nonlinear operators do not become self-adjoint
elliptic operators.
\end{corollary}

\end{document}